\documentclass{article}

\usepackage[preprint]{neurips_2024}
\usepackage{multirow}

\usepackage[utf8]{inputenc} 
\usepackage[T1]{fontenc}    
\usepackage{hyperref}       
\usepackage{url}            
\usepackage{booktabs}       
\usepackage{amsfonts}       
\usepackage{nicefrac}       
\usepackage{microtype}      
\usepackage{xcolor}         
\usepackage{graphicx}
\usepackage{wrapfig}
\usepackage{subfigure}
\usepackage{bm}

\usepackage[ruled,vlined]{algorithm2e}

\usepackage{amsmath}
\usepackage{amssymb}
\usepackage{mathtools}
\usepackage{amsthm}

\usepackage{algorithmic}
\usepackage[capitalize,noabbrev]{cleveref}

\theoremstyle{plain}
\newtheorem{theorem}{Theorem}[section]

\newtheorem{lemma}[theorem]{Lemma}

\theoremstyle{definition}
\newtheorem{definition}[theorem]{Definition}

\theoremstyle{remark}

\usepackage[textsize=tiny]{todonotes}

\usepackage{soul,color,xcolor}      
\definecolor{myColor}{rgb}{0.8039,0,0}   
\makeatletter
\newcommand*{\new}{\@ifnextchar\bgroup{\new@}{\color{blue}}}
\newcommand*{\new@}[1]{{\textcolor{blue}{#1}}}
\makeatother

\title{Learning from Random Demonstrations: Offline Reinforcement Learning with Importance-Sampled Diffusion Models}

\author{%
  Zeyu Fang \\
  George Washington University\\
  \texttt{joey.fang@gwu.edu} \\
  \And
  Tian Lan \\
  George Washington University\\
  \texttt{tlan@gwu.edu} \\
}

\begin{document}

\maketitle

\begin{abstract}
Generative models such as diffusion have been employed as world models in offline reinforcement learning to generate synthetic data for more effective learning. Existing work either generates diffusion models one-time prior to training or requires additional interaction data to update it. In this paper, we propose a novel approach for offline reinforcement learning with closed-loop policy evaluation and world-model adaptation. It iteratively leverages a guided diffusion world model to directly evaluate the offline target policy with actions drawn from it, and then performs an importance-sampled world model update to adaptively align the world model with the updated policy. We analyzed the performance of the proposed method and provided an upper bound on the return gap between our method and the real environment under an optimal policy. The result sheds light on various factors affecting learning performance. Evaluations in the D4RL environment show significant improvement over state-of-the-art baselines, especially when only random or medium-expertise demonstrations are available -- thus requiring improved alignment between the world model and offline policy evaluation.

\end{abstract}

\section{Introduction}

Offline Reinforcement Learning (RL) methods have received much recent attention~\cite{levine2020offline, prudencio2023survey}, due to their abilities to train policies based on offline datasets (also known as demonstrations) that are collected using a behavior policy, rather than through expensive (and sometimes dangerous) online interactions~\cite{kiran2021deep}. The available datasets can only include limited and fixed transitions/trajectories. Thus, as the learned policy gradually deviates from the behavior policy used for collecting data, the learned policy and the estimated value function would 
overestimate out-of-distribution actions in unseen dynamics. This is known as the distribution shift~\cite{kumar2019stabilizing}, one of the key challenges for offline RL.

To this end, solutions have been proposed to augment the offline training dataset by building world models~\cite{yu2021combo,rigter2022rambo, matsushima2020deployment}, besides the approach of policy regularization~\cite{kumar2019stabilizing, rashidinejad2021bridging}. The idea is to learn a synthetic model, known as the world model, representing the transition dynamics of the underlying Markov Decision Process (MDP). Once trained with the offline dataset, the world model can be leveraged to generate additional synthetic trajectories for offline RL. Existing work include world models generated by VAE~\cite{ha2018world, hafner2023mastering, ozair2021vector}, GAN~\cite{eysenbach2022mismatched}, Transformers~\cite{janner2021offline}, and more recently Diffusion~\cite{Ding2024DiffusionWM, lu2023synthetic}. However, most existing work either obtain a world model one-time prior to policy training~\cite{yu2020mopo, kidambi2020morel, janner2019trust} or require additional online interaction data to adapt the world model~\cite{kaiser2019model, hafner2019dream}, neither effectively mitigating distribution shift with only offline data. Further, the return gap between world models (e.g., diffusion models) and the real environment of such offline RL algorithms remains to be analyzed.

\begin{figure}\label{fig:intro}
    \centering
    \includegraphics[width=0.8\columnwidth]{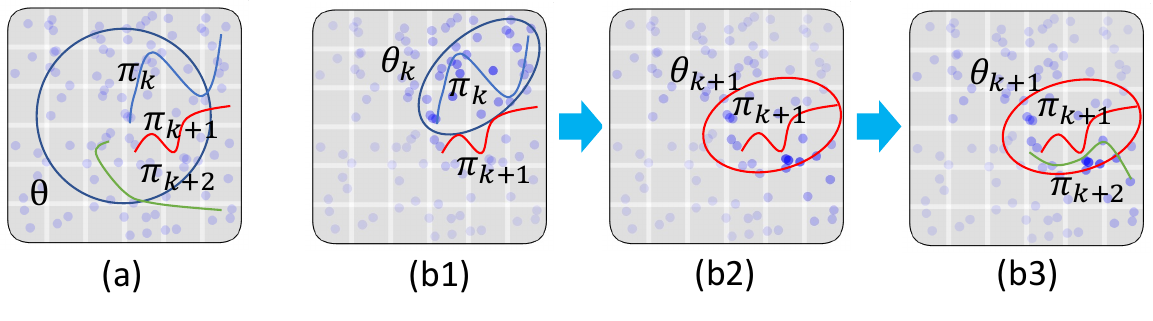}
    \caption{The existing algorithms as shown in (a) use a fixed world model in offline RL (or require additional data from the environment to update). Our proposed approach, as illustrated in (b1) to (b3), adapts the diffusion model from $\theta_k$ to $\theta_{k+1}$ with importance-sampling regarding the updated policy $\pi_{k+1}$, and then evaluates the current policy $\pi_{k+1}$ with guided diffusion to obtain $\pi_{k+2}$. Each point refers to a sample in the offline dataset, while darker points are given higher weight in loss calculation for diffusion world model update. The performance of the proposed method is analyzed in this paper.}
    \label{fig:enter-label}
\end{figure}

This paper proposes a novel approach for offline RL with closed-loop policy evaluation and world-model adaptation. Our approach encompasses two collaborative components: (i) A guided diffusion world model to directly evaluate the target policy by generating synthetic trajectories with actions drawn from it; and (ii) An importance-sampled world model update to align the world model with the target policy. The two components work in a closed-loop operation throughout training, as illustrated in Figure~\ref{fig:intro}. Existing work often considers diffusion world models as a data synthesizer~\cite{lu2023synthetic} or planner~\cite{janner2022planning, Ding2024DiffusionWM}. It generates additional synthetic trajectories, following data collected using a behavior policy, to support policy updates as shown in Figure~\ref{fig:intro}(a). Our proposed approach, on the other hand, continually adapts the diffusion model using importance sampling with respect to the distribution shift between the current target policy $\pi_k$ and the behavior policy $\pi_b$. The updated diffusion model is then used to evaluate $\pi_t$ for policy improvement, with a sequence of actions drawn from $\pi_k$ used as guidance/input to the world model. We note that this guided policy evaluation requires iterating between next-state generation using the guided diffusion world model and next-action sampling from current target policy. It is different from previous work using diffusion on trajectory synthesize ~\cite{lu2023synthetic}, planning ~\cite{janner2022planning, Ding2024DiffusionWM}, or policy representation~\cite{wang2022diffusion, chen2022offline} in RL. 

We analyze the performance of the proposed learning algorithm, namely Adaptive Diffusion World-Model for Policy Evaluation (ADEPT). It addresses the distribution shift problem in offline RL, especially when only random and medium-expertise demonstrations are available. We provide the bound of value discrepancy between actual environment and our diffusion model under the same policy, and show that the monotonic improvement can be guaranteed when the one-step policy update under the model is larger than this bound. We further decompose the bound into three factors: state transition error, reward prediction error and policy shift, and discuss how ADEPT lowers these factors to narrow the bound. To our best of knowledge, this is the first analysis for offline RL with diffusion world models. We note that ADEPT works best with datasets consisting of mainly random and medium demonstrations, where the distribution shift becomes more severe as target policy moves toward optimum. The proposed algorithm is evaluated on D4RL benchmark \cite{fu2020d4rl} in three MuJoCo environments. The results show that ADEPT improves the baseline IQL method \cite{kostrikov2021offline} with an average of 211.8\% on random, 15.6\% on medium and 19.4\% on medium-replay, and outperforms other SOTA offline RL methods including diffusion-based algorithms.

The contributions of this work can be summarized as follows:
\begin{enumerate}
    \item We propose an offline RL algorithm, ADEPT, for a closed-loop operation of policy evaluation with guided diffusion and world-model adaptation with importance sampling. 
    \item We provide theoretical proof of bounding the return gap between ADEPT (using offline RL and diffusion world models) and actual environment under an optimal policy. 
    \item We evaluate our method on the D4RL benchmarks and demonstrate significant improvement over state-of-the-art baselines, especially on random and medium datasets.
\end{enumerate}

\section{Related Works}


\paragraph{Offline RL.} Offline RL faces the distribution shift problem~\cite{kumar2020conservative} due to data collected using a specific behavior policy.
Various methods have been proposed to regularize an offline RL policy and address this issue. In particular,
MOReL \cite{kidambi2020morel}, MOPO \cite{yu2020mopo}, VI-LCB \cite{rashidinejad2021bridging} and COMBO \cite{yu2021combo} develop different ways to penalize reward or value function in unseen state and action pairs to address the out-of-distribution issues. Other model-free methods, including BCQ~\cite{fujimoto2019off}, IQL~\cite{kostrikov2021offline}, CQL~\cite{kumar2020conservative}, and TD3+BC~\cite{fujimoto2021minimalist}, develop different conservatism mechanisms that are defined on action or value function, forcing the policy to act more conservatively. Our proposed ADEPT framework can be combined with any of these offline RL algorithms.

\paragraph{World Models for Offline RL.}
The use of world models to generate synthetic data for offline RL 
was first proposed in \cite{ha2018world}, utilizing VAE and RNN for predicting state transitions. Following this approach, various world models with advanced capabilities of fitting desired distributions have been proposed, including convolutional U-networks \cite{kaiser2019model}, vector-quantized autoencoders \cite{ozair2021vector}, generative adversarial networks \cite{eysenbach2022mismatched}, energy-based models \cite{boney2020regularizing}, Transformers \cite{janner2021offline}, and diffusion~\cite{Ding2024DiffusionWM, lu2023synthetic}. They mainly use world models for trajectory synthesis, with limited adaptability of the world models (except for using additional online data~\cite{lu2023synthetic}).

\paragraph{The Use of Diffusion Models in RL.}
Diffusion is a state-of-the-art technique for 
generating synthetic samples of images and text data \cite{ho2020denoising}. It was first introduced as planner or policy representation in offline RL~\cite{janner2022planning},
where the diffusion model directly generates trajectories that are used for execution.
This is further extended to conditional actions~\cite{ajay2022conditional}, meta-RL~\cite{ni2023metadiffuser}, hierarchical tasks~\cite{li2023hierarchical}, multi-task problems~\cite{he2023diffusion}, multi-agent tasks~\cite{zhu2023madiff} and safe planning~\cite{xiao2023safediffuser}.
Diffusion models are also employed for policy expression~\cite{wang2022diffusion, chen2022offline}, imitation learning~\cite{hegde2024generating}
and reward modeling~\cite{nuti2023extracting}. Diffusion models are adopted as a data synthesizer to generate additional synthetic data based on offline datasets before policy training~\cite{lu2023synthetic}. Later, a conditional diffusion world model was proposed to generate trajectories from current state and action, to support offline value-based RL~\cite{Ding2024DiffusionWM}. Different from these existing works, we propose ADEPT for closed-loop policy evaluation and world-model adaptation, and provide the theoretical analysis of the return gap between ADEPT and actual environment. 


\section{Preliminary}



\paragraph{Offline RL using World Models}
We consider an unknown Markov Decision Process (MDP), referred to as the environment.
Supposing the MDP is fully-observable with discrete time, it could be defined by the tuple $M = (S, A, P, R, \mu_0, \gamma)$. 
$S$ and $A$ are the state and action spaces, respectively.
$P(s_{t+1}|s_t,a)$ is the transition probability and $R: S \times A \to \mathbb{R}$ is the reward function.
$\mu_0$ is the initial state distribution and $\gamma$ is the discount factor.
We consider an agent acts within the environment based on a policy $\pi(a|s)$ repeatedly. In each time step, the agent receives a state $s_t$ and samples an action via its policy $a_t \sim \pi( \cdot | s_t)$.
The environment transits into a new state $s_{t + 1} \sim P( \cdot | s_t, a_t )$ and returns a reward $r_t = R(s_t, a_t)$.
After a whole episode of interactions, a trajectory $\tau = (s_0, a_0, r_0, s_1, \dots, s_T, a_T, r_T)$ will be generated, which contains states, actions and rewards of maximum length $T$.
Based on that, the goal of RL is to learn an optimal policy $\pi^*$ to maximize the expectation of cumulative rewards from this MDP: $\pi^* = \arg\max_{\pi} \mathbb{E}_{s, a \sim \pi}(\sum_{t=0}^{T} \gamma^t r_t)$.

Specifically in offline model-based RL, only a dataset of trajectories $\mathcal{D}$ is available. Therefore, a prediction model of the environment is introduced, denoted as world model $\hat{M}$, to improve sample efficiency for further learning and planning.
Commonly the world model learns a single-step transition approximating the real dynamics $M$ of the environment in a supervised method based on $\mathcal{D}$.
Hence, once a world model has been trained, it could replace the real environment to generate synthetic trajectories.
Similar to standard RL, an initial state $s_0$ is sampled first from datasets, and based on that the interactions start.
After certain length of steps $H$, referred as horizon, a synthetic trajectory $\hat{\tau} = (\hat{s}_0, \hat{a}_0, \hat{r}_0, ..., \hat{s}_H, \hat{a}_H, \hat{r}_H)$ is generated, in which $\hat{s}_{t + 1} \sim \hat{P}( \cdot | \hat{s}_t, \hat{a}_t )$ and $\hat{r}_t = \hat{R}(s_t, a_t)$.
These imaginary trajectories are added into the experience buffer for RL algorithms to optimize its policy.

\paragraph{Diffusion Model} \label{preliminarydiffusion}
Diffusion model is one of the SOTA deep generative models for images and videos.
The purpose of diffusion model is to learn an underlying data distribution $q(\bm{x_0})$ from a dataset $\mathcal{D} = \left\{ \bm{x_i} \right\}$.
In Denoising Diffusion Probabilistic Model (DDPM) \cite{nichol2021improved}, the synthetic data generation is conducted by denoising real data $x_0$ from noises $\mathcal{N}(\bm{0}, \bm{I})$ with $K$ steps. 
To train such a denoising step, we first formulate a predefined forward process $q\left(\bm{x_{k}} | \bm{x_k-1}\right) = \mathcal{N}\left(\sqrt{\alpha_{k-1}}\bm{x_{k-1}}, \sqrt{1-\alpha_{k-1}}\bm{I}\right)$ to add noise on real data in the dataset step by step.
Finally the distribution would be similar to Gaussian noises.
Here, $\alpha_k = 1 - \beta_k$, and $\beta_{1:K}$ is a certain various schedule.
The diffusion model define another parameterized reverse process $p_\theta\left(\bm{x_{k-1}} | \bm{x_k}\right) = \mathcal{N}\left(\mu_\theta(\bm{x_k}), \bm{\Sigma_k}\right)$ to denoise the real data from the Gaussian noise $\mathcal{N}\left(0, \bm{I}\right)$.
By defining $\Bar{\alpha}_k = \prod_{i=1}^k \alpha_i$, $\mu_\theta$ and $\bm{\Sigma_k}$ can be rewritten as follows:
\begin{align} \label{denoise}
    \mu_\theta(\bm{x_k}) = \frac{1}{\sqrt{a_k}}\left(\bm{x_k} - \frac{\beta_k}{\sqrt{1 - \Bar{\alpha}_k}} \bm{\epsilon}_\theta(\bm{x_k}, k)\right) \ \ {\rm and} \ \ 
    \bm{\Sigma_k} = \beta_k \frac{1 - \Bar{\alpha}_{k-1}}{1 - \Bar{\alpha}_k}\bm{I}.
\end{align}
Here, $\bm{\epsilon}_\theta$ is the parameterized noise prediction model to be trained.
Therefore, the loss function of the diffusion model is defined as: $L(\theta) = \mathbb{E}_{k \sim [1,K], x_0 \sim q, \bm{\epsilon} \sim N(\bm{0}, \bm{I})}(\lVert\bm{\epsilon} - \bm{\epsilon}_\theta(\bm{x}_k, k)\rVert^2)$,
where $\bm{\epsilon}$ is the real noise added in each step.
When generating the synthetic data $\hat{x_0}$, beginning with $\hat{x}_K$ sampled from $\mathcal{N}(\bm{0}, \bm{I})$,  predicts noise possibly added in each step until reaching $\hat{x}_0$.

In this work, the diffusion model is guided by current state and action to predict the next state.
Such a setting leads to modeling a conditional distribution $q(\bm{x}|\bm{y})$, in which $y$ is the attribute or label of the generated samples.
To handle this, two kinds of conditional diffusion models have been proposed: classifier-guided and classifier-free methods.
In classifier-guided methods, an extra parameterized classifier $p_\phi(\bm{y}|\bm{x_k})$ is trained on noisy sample of $\bm{x}$ and label $\bm{y}$.
On the contrary, classifier-free methods don't train a separate classifier, but regard the condition as an extra input of the noise prediction model.
In that case, the noise prediction equation is replaced by:
\begin{equation} \label{classifier_free_equation}
\bm{\epsilon}_\theta(\bm{x_k}, k | \bm{y}) = \bm{\epsilon}_\theta(\bm{x_k}, \varnothing, k) + \omega (\bm{\epsilon}_\theta(\bm{x_k}, \bm{y}, k) - \bm{\epsilon}_\theta(\bm{x_k}, \varnothing, k)),
\end{equation}
in which $\varnothing$ is the default value representing no conditions.
In ADEPT, we choose the classifier-free method to train the diffusion model.

\section{Methodology}

\begin{figure}
    \centering
    \includegraphics[width=0.9\columnwidth]{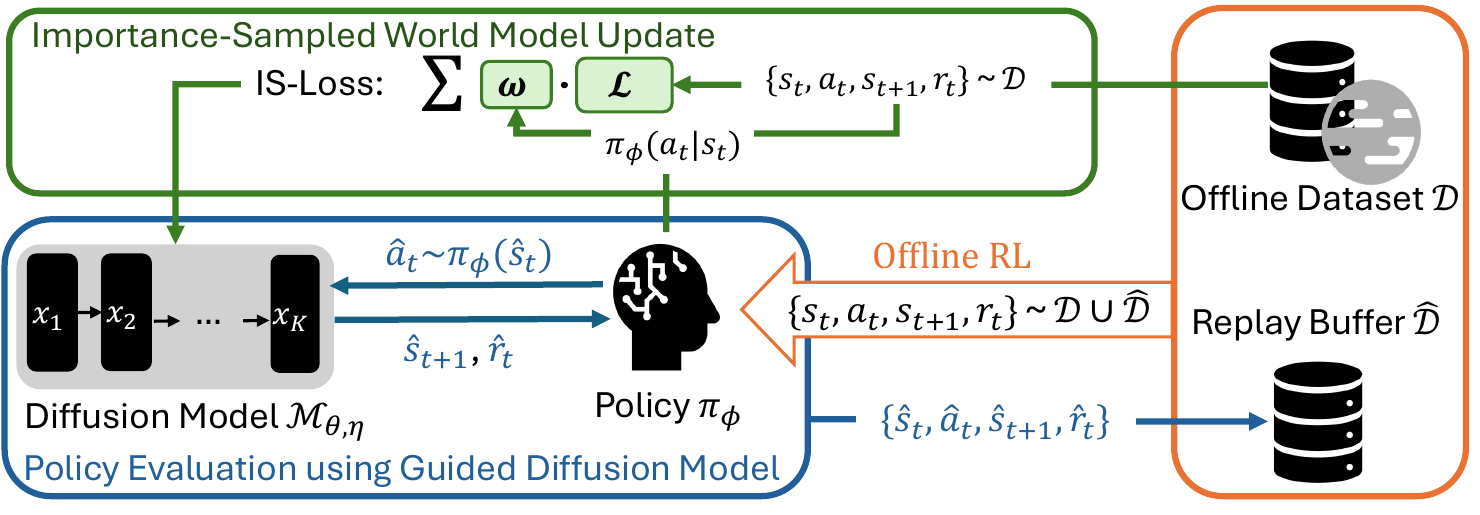}
    \caption{An overview of our ADEPT algorithm. It iteratively leverages a guided diffusion world model to directly evaluate the target policy with actions drawn from it, and then performs an importance-sampled world model update to adaptively align the world model with the updated policy. An upper bound on the return gap is analyzed in Section~4.3.}
    \label{fig:1}
\end{figure}


In this section, we illustrate details of ADEPT. As shown in Figure \ref{fig:1}, the two key components in ADEPT: policy evaluation on guided diffusion world model, and importance-sampled world model update, work in a closed-loop operation throughout the whole training process. The explanations of these two components are covered in \ref{method1} and \ref{method2}, respectively.
In \ref{proof}, we provide theoretical derivation to bound the return gap of the proposed method.

\subsection{Policy Evaluation using Guided Diffusion Model} \label{method1}

Before the policy iteration, we first utilize the offline dataset to initialize guided diffusion world model $\mathcal{M}_{\theta, \eta}$, consisting of noise prediction model $\bm{\epsilon}_\theta$ and a multi-layer perceptron $r_\eta$, to simulate the conditional distribution of the one-step transition $P(s_{t+1}|s_t, a_t)$ and reward function $R(s_t, a_t)$. 
The offline dataset is normalized by linearly mapping each dimension of the state space to have $0$ mean and $1$ standard deviation.
While training the diffusion model, we sample a minibatch of tuples $(s_t, a_t, r_t, s_{t+1})$ consisting of the state, action, reward and next state from the normalized offline datasets in each iteration.
As mentioned in \ref{preliminarydiffusion}, we follow DDPM and adopt classifier-free method to learn the conditional probabilities of the state transition based on current state $s_t$ and action $a_t$, \emph{i.e.}, we replace $x$ with $s_{t+1}$ and $y$ with $(s_t, a_t)$ in equation \eqref{classifier_free_equation}.
Since the reward function plays a significant role in RL training, the model should emphasize more on its accuracy and decouple reward function from state transition dynamics.
Therefore, we separately train $r_{\eta}(s_t, a_t, s_{t+1})$ to predict the reward function and the terminal signal. Introducing $s_{t+1}$ as an extra input significantly improves the accuracy of terminal signal prediction, and in evaluation the next state is sampled and denoised by $\bm{\epsilon}_\theta$.

Once $\mathcal{M}_{\theta, \eta}$ is initialized, we could use it to interact with the current policy $\pi_\phi$, generating data to evaluate the current policy. At the beginning of each policy evaluating iteration,  $\mathcal{M}_{\theta, \eta}$ randomly samples a state from $\mathcal{D}$ to be the start state $\hat{s}_0$, even though it may appear as a middle state in the real trajectory. Based on the current state $\hat{s}_t$, an action $\hat{a}_t \sim \pi_\phi(\hat{s}_t)$ is sampled given target policy. Conditioned by $\hat{s}_t$ and $\hat{a}_t$, the diffusion model generates the next state $\hat{s}_{t+1}$ after $K$ denoising steps via $\bm{\epsilon}_\theta$. Next, $r_\eta$ generates the reward $\hat{r}_t$ and terminal signal $d_t$, given $\hat{s_t}, a_t$, and $\hat{s}_{t+1}$, in which $\hat{s}_{t+1}$ and $\hat{r_t}$ are clipped into the range of the offline dataset. Such an iteration continues till when $d_t$ is true or $t$ reaches horizon $H$. The generated trajectory $\hat{\tau}$
is added into the replay buffer $\hat{\mathcal{D}}$, combined with $\mathcal{D}$ to improve the policy via offline RL algorithms.

\subsection{Importance-Sampled World Model Update} \label{method2}

Once the policy is updated, there is a policy shift between $\pi_\phi$ and the behavior policy $\pi_\mathcal{D}$ that collects $\mathcal{D}$. The distribution estimation from $\mathcal{M}_{\theta, \eta}$ could lose accuracy under the new policy. To handle this, we adopt the importance-sampling technique to update $\mathcal{M}_{\theta, \eta}$ with offline dataset, guiding $\hat{\mathcal{D}}$ towards the accurate distribution under the current policy. 
This is achieved by re-weighting the loss function of multiple samples to reduce the discrepancy between $\pi_\phi$ and $\pi_\mathcal{D}$. 
Even if $\pi_\mathcal{D}$ is not available, it's not hard to estimate the behavior policy from the offline dataset via behavior cloning (BC)~\cite{nair2018overcoming}.
For each transition $\{(s^i_{t_i}, a^i_{t_i}, r^i_{t_i}, s^i_{t_i+1})\}_{i=1}^N$ in the training batch, given their loss as $l_i(\theta, \eta)$, the importance weight as $\omega_i$, the total loss $\mathcal{L}(\theta, \eta)$ of the whole training batch is calculated as:
\begin{equation} \label{diffloss}
    \mathcal{L}(\theta, \eta) = \frac{1}{N} \sum_{i=1}^{N} \omega_i l_i(\theta, \eta) = \frac{1}{N}\sum_{i=1}^{N} \frac{\pi_\phi(a^i_{t_i} | s^i_{t_i})}{\pi_\mathcal{D}(a^i_{t_i} \vert s^i_{t_i})} l_i(\theta, \eta).
\end{equation}
Generally, the state-action pairs that have more probabilities under the current policy are associated with a larger weight in loss calculation.

\begin{algorithm} \label{algo1} \small
\caption{Our proposed ADEPT algorithm}

    \textbf{Hyperparameters:} denoising steps $K$, schedule $\beta_{1:K}$, evaluation steps per epoch $N_{e}$, batchsize for policy improvement $B_p$, Batchsize for world model update $B_m$, Horizon $H$

    \textbf{Input:} offline dataset $\mathcal{D}$, diffusion world model $\mathcal{M}_{\theta, \eta}$
    
    Initialize target policy $\pi_\phi$, replay buffer $\mathcal{\hat{D}} = \varnothing$;
    Normalized the dataset $\mathcal{D}$

    Initialize $\mathcal{M}_{\theta, \eta}$ with $\mathcal{D}$ and $\pi_\mathcal{D}$ till convergence: $IWU(\pi_\mathcal{D}, \mathcal{D}, \theta, \eta)$
    
    \While{not converged}{

    \For{$j= 0 \to N_{e}$}
    {
    
    $\mathcal{\hat{D}} \leftarrow \mathcal{\hat{D}} \cup PE(\mathcal{M}_{\theta, \eta}, \pi_\phi)$
    }
    Sample $\mathcal{B} = \{(s^i_{t_i}, a^i_{t_i}, r^i_{t_i}, s^i_{t_i+1})\}_{i=1}^{B_p}\sim\mathcal{D} \cup \mathcal{\hat{D}}$
    
    Update $\phi$ with $\mathcal{B}$ via offline RL methods

    $IWU(\pi_\phi, \mathcal{D}, \theta, \eta)$
    }

   \hrulefill \\
   \vspace{0.05in}
\underline{\bf Subroutine:} Importance-Sampled World-Model Update $(IWU)$:
    
    Sample batch $\{(s^i_{t_i}, a^i_{t_i}, r^i_{t_i}, s^i_{t_i+1})\}_{i=1}^{B_m}\sim\mathcal{D}$

        \For{$i= 0 \to B_m$}{
        
        $k \sim Uniform(\{1, 2, \dots, K\})$; $\bm{\epsilon} \sim \mathcal{N}(\bm{0}, \bm{I})$

        $s_{noise} = \sqrt{\Bar{\alpha}_k} s^i_{t_i+1} + \sqrt{1-\Bar{\alpha}_k}\bm{\epsilon}$

        Get Importance-sampling weight $\omega_i$ under $\pi$ via Equation \eqref{diffloss}
        
        $l_i = \lVert\bm{\epsilon} - \bm{\epsilon}_\theta(s_{noise}, (s^i_{t_i}, a^i_{t_i}), k) \rVert^2 + \lVert r_t - r_{\eta}(s^i_{t_i}, a^i_{t_i}, s^i_{t_i+1})\rVert^2$
        }
    Calculate $\mathcal{L}(\theta, \eta)$ via Equation \eqref{diffloss} and take gradient step on it.

    \hrulefill \\
   \vspace{0.05in}
    \underline{\bf Subroutine:} Policy Evaluation using Guided Diffusion Model ($PE$):
    
    Randomly select a state from $\mathcal{D}$ as $\hat{s}_0$.

    \For{$t=0 \to H$}
    {
    $x_K \sim \mathcal{N}(\bm{0}, \bm{I});\hat{a}_t \sim \pi(\hat{s}_t)$
    
    \For{$k=K-1 \to 0$}
    {
    
        $x_k \sim \mathcal{N}(\mu_\theta(x_{k+1} | \hat{s_t}, \hat{a}_t), \bm{\Sigma_k})$, in which $\mu_\theta$ and $\bm{\Sigma_k}$ are defined in Eq. \eqref{denoise}.
    }

    $\hat{s}_{t+1}$ = $x_0; \hat{r}_t, d_t = r_\eta(\hat{s}_t, \hat{a}_t, \hat{s}_{t+1})$

    \textbf{If} $d_t == True$ \textbf{then} break
    }
    Return the trajectory $\hat{\tau} = (\hat{s}_0, \hat{a}_0, \hat{r}_0, \hat{s}_1, \dots \hat{s}_t, \hat{a}_r, \hat{r}_t, \hat{s}_{t+1})$
\end{algorithm}
The complete training procedure of ADEPT is illustrated in Algorithm \ref{algo1}. Two subroutines: policy evaluation and world-model update alternate iteratively until convergence. In this work we select the state-of-the-art algorithm IQL \cite{kostrikov2021offline} and traditional method SAC \cite{haarnoja2018soft} as the offline RL methods to show the performance. The selection of hyperparameters is a tradeoff between performance and efficiency, which will be discussed and analyzed in the appendix.

\subsection{Return Gap Analysis} \label{proof}

To show that the expected return improvement could be guaranteed when adopting diffusion model as the world model to train the policy, we wish to provide a lower bound $C$ of the returns of the policy:
\begin{equation}    
J(\pi) = \mathbb{E}_{s_t,a_t \sim \pi} \sum_{t=0}^T \gamma^t r_t.
\end{equation}
By optimizing the lower bound, the performance could be guaranteed non-decreasing:
\begin{equation}
    J(\pi) \geq \hat{J}(\pi) - C.
\end{equation}

Here, $J(\pi)$ and $\hat{J}(\pi)$ are the returns of the policy under the real environment and the world model used for training and evaluating, which is the diffusion model in this work. With such a condition, we could guarantee improvement under the actual environment if the returns under the simulating model are promoted by at least $C$. Furthermore, $C$ is expected to be expressed in terms of the error quantities of the model. We denote the reward prediction error as $\hat{\varepsilon}_r$, the model transition error as $\hat{\varepsilon}_m$, and policy distribution shift error as $\hat{\varepsilon}_\pi$. Their detailed definitions are presented as follows:

\begin{definition} \label{definition_1}
We define $\hat{\varepsilon}_r$ to be the maximal expectation of total-variation distance (TV-distance) of the probabilities between predicted reward and true reward under the target policy $\pi_\phi$.
\begin{equation}
    \max_t \mathbb{E}_{s_t,a_t \sim \pi_{\phi}} D_{TV}(P(r_t\vert s_t,a_t) \Vert P_\eta(r_t \vert s_t,a_t)) \leq \hat{\varepsilon}_r.  \label{assump_0}
\end{equation}
\end{definition}
Although in practice the reward prediction model also takes $s_{t+1}$ as an input, it can still be considered only conditioned by $s_t$ and $a_t$, since $s_{t+1}$ is generated from denoising steps.
\begin{definition} \label{definition_2}
Similar to the Definition \ref{definition_1}, $\hat{\varepsilon}_m$ is defined as the maximal expected TV-distance of the probabilities between predicted next state and true value under $\pi_\mathcal{D}$.
\begin{equation}
    \max_t \mathbb{E}_{s_t,a_t \sim \pi_{\phi}} D_{TV}(P(s_{t+1}\vert s_t,a_t \Vert P_\theta(s_{t+1} \vert s_t,a_t)) \leq \hat{\varepsilon}_m,
\end{equation}
\end{definition}
In practice, $\hat{\varepsilon}_r$ and $\hat{\varepsilon}_m$ could be estimated by measuring the mean square error (MSE) between the predicted states or rewards and their values in the offline dataset in practice. To be noticed that $s_t$ and $a_t$ are sampled from the distribution under $\pi_\phi$ in both equations, but the diffusion model can only be trained from the offline dataset collected by $\pi_\mathcal{D}$. Therefore, it's necessary to take importance-sampling technique to keep the diffusion model fitting the new distribution under $\pi_\phi$.

\begin{definition} \label{definition_3}
$\hat{\varepsilon}_\pi$ is denoted by the maximal TV-distance of the policy between iterations.
\begin{equation}
\max_s D_{TV}(\pi(a \vert s)\Vert \pi_\phi(a\vert s)) \leq \hat{\varepsilon}_\pi.
\end{equation}
\end{definition}
This error measures how the evaluating policy $\pi$ has shifted from the target policy $\pi_\phi$ that is evaluated under diffusion world model and collects the dataset for policy improvement. 

\begin{theorem} Given $\hat{\varepsilon}_r$, $\hat{\varepsilon}_m$ and $\hat{\varepsilon}_\pi$, the bound $C$ between the true return and the IS-diffusion model-based return can be expressed as follows:
\begin{equation}
    J(\pi) \geq \hat{J}_{\theta, \eta}(\pi) - 2r_{max}(\frac{\hat{\varepsilon}_r + 2\hat{\varepsilon}_\pi}{1 - \gamma} + \frac{\gamma(2\hat{\varepsilon}_\pi + \hat{\varepsilon}_m)}{(1-\gamma)^2}).
\end{equation}
\end{theorem}
\begin{proof}
We give an outline of our proof next and display the whole proof in the Appendix. To prove this, we first divide the discrepancy between $J(\pi)$ and $\hat{J}_{\theta, \eta}(\pi)$ into $\mathcal{L}_1 = J(\pi) - J(\pi_\phi)$, $\mathcal{L}_2 = J(\pi_\phi) - \hat{J}_{\theta,\eta}(\pi_\phi)$ and $\mathcal{L}_3 = \hat{J}_{\theta,\eta}(\pi_\phi) - \hat{J}_{\theta, \eta}(\pi)$. Then we will prove several lemmas to show that if the reward is finite, $\mathcal{L}_1$ and $\mathcal{L}_3$ could be bounded by $\hat{\varepsilon}_\pi$, while $\mathcal{L}_2$ could be bounded by $\hat{\varepsilon}_m$ and $\hat{\varepsilon}_r$. Finally we combined these three bounds together to complete the theorem.
\end{proof}

\begin{lemma}
    We define $r_{max} = \max_{s,a} r(s,a)$. For any $\pi, \hat{\pi}$ under any two environmental model $\mathcal{M}$ and $\hat{\mathcal{M}}$, satisfying
    \begin{equation}
        \max_t \mathbb{E}_{s_t,a_t \sim \pi} D_{TV}(P(r_t\vert s_t,a_t) \Vert \hat{P}(r_t \vert s_t,a_t)) \leq \varepsilon_r ,
    \end{equation}
    we have 
    \begin{equation}
        \left\vert J(\pi) - \hat{J}(\hat{\pi}) \right\vert
        \leq r_{max} \sum_t \gamma^t (2\varepsilon_r + 2D_{TV}(P_t(s,a)\Vert \hat{P}_t(s,a))).
    \end{equation}
\end{lemma}
Here, $P_t(s,a)$ denotes the joint distribution of state and action at time step $t$ under $\mathcal{M}$ and $\pi$, while $\hat{P}_t(s,a)$ is under $\hat{\mathcal{M}}$ and $\hat{\pi}$. This lemma decomposes the discrepancy into the form of reward error and the state-action joint distribution shift. Next, we further analyze the bound of the latter.

\begin{lemma} \label{lemma2}
    For $P_t$ under $\pi$ and $\mathcal{M}$, and $\hat{P}_t$ under $\hat{\pi}$ and $\hat{\mathcal{M}}$, there existing $\delta$ \emph{s.t.}
    \begin{equation}
        \max_t \mathbb{E}_{s,a \sim \pi}D_{TV}(P_{t+1}(s',a' \vert s_t = s , a_t = a) \Vert \hat{P}_{t+1}(s',a'\vert s_t = s, a_t = a)) \leq \delta.
    \end{equation}
    Then we have
    \begin{equation}
        D_{TV}(P_t(s, a) \Vert \hat{P}_t(s, a)) \leq t\delta + D_{TV}(P_0(s,a)\Vert \hat{P}_0(s,a)).
    \end{equation}
\end{lemma}

This lemma provides the upper bound of the state-action joint distribution shift at time step $t$, when the initial state distribution is shared and the state transition distribution under different policies and models is bounded. Next step is to decompose the joint distribution into the form of $\varepsilon_m$ and $\varepsilon_\pi$. 

\begin{lemma} \label{lemma3}
For any 2 different joint distributions of state and action pairs $P$ and $\hat{P}$, we have
\begin{equation}
    D_{TV}(P(s, a)\Vert \hat{P}(s,a)) \leq D_{TV} (P(s)\Vert \hat{P}(s)) + \max_s D_{TV}(\pi(a \vert s)\Vert \hat{\pi}(a\vert s)).
\end{equation}
\end{lemma}

To be noticed that this lemma is also correct for conditional distribution of $s$ and $a$. By combining Lemma \ref{lemma2} and Lemma \ref{lemma3}, the joint distribution shift can be expressed with $\varepsilon_r$ and $\varepsilon_m$. Therefore, when applied with the error defined in diffusion model, the final bound $C$ could be concluded.

With this theorem, the monotonic improvement of the true return $J(\pi)$ is guaranteed when the returns under IS-diffusion model $\hat{J}_{\theta, \eta}(\pi)$ is improved by more than $C$.
Besides, by lowering $\hat{\varepsilon}_r$, $\hat{\varepsilon}_m$, and $\hat{\varepsilon}_\pi$, this bound could be further narrowed. These errors are mainly caused by the following reasons:

\begin{itemize}
 \vspace{-0.05in}
    \item \textbf{The compounding error}: Though the diffusion model has lower prediction error compared with conventional world models, as a one-step transition model, the variation in each step accumulates as the compounding error, which significantly increases $\hat{\varepsilon}_m$.
    \item \textbf{Out of distribution}: While using diffusion world model for policy evaluation, the action derived from the policy can drive the state out of distribution. In that case, the generated state and reward become unstable or even illegal, causing a high value of $\hat{\varepsilon}_m$ and $\hat{\varepsilon}_r$.
    \item \textbf{Policy shifting}: While the target policy keeps getting improved with both the offline dataset and generated replay buffer, it's drifting away from the original behavior policy that collects the data, leading to a high value of $\hat{\varepsilon}_\pi$.
    \vspace{-0.05in}
\end{itemize}

To handle these, in our method $H$ is set to be small enough to limit the accumulating variations. Since the initial state is randomly chosen, the whole dataset can be covered even with small horizons. The generated states and rewards are clipped within the range of the dataset to avoid out-of-distribution. Finally to solve the distribution shift problem, inspired by the model adaption to evolving policy method in online RL \cite{wang2023live}, we adopt importance-sampling in offline RL to continuously align the diffusion model $\mathcal{M}_{\theta, \eta}$ with the new distribution under current policy $\pi_\phi$ to lower $\hat{\varepsilon}_\pi$.

\section{Experiment}

Our experiments are designed to evaluate:
1. The performance of ADEPT with adaptive diffusion world model and offline RL updates, compared with other SOTA algorithms including diffusion-based methods.
2. The effectiveness of the proposed importance sampling and policy improvement in ADEPT.
We train and test our method on multiple environments and datasets in D4RL \cite{fu2020d4rl} to show the quantitative results, and further analyze our method with ablation study.




\subsection{Numerical Evaluation}

In this section, we evaluate our proposed ADEPT with 2 widely-used offline RL methods: IQL\cite{kostrikov2021offline} and SAC\cite{haarnoja2018soft} over the MuJoCo locomotion environments (e.g., halfcheetah, walker2d, and hopper) and on 3 different datasets (random, medium and medium-replay).
We select a number of SOTA algorithms as baselines, including model-free methods TD3+BC \cite{fujimoto2021minimalist}, CQL \cite{kumar2020conservative}, model-based methods such as MOPO \cite{yu2020mopo}, COMBO \cite{yu2021combo}, and diffusion-based methods as SyntheER \cite{lu2023synthetic} and Diffuser \cite{janner2022planning}, as well as behavior cloning. All experiments are conducted with the same training hyperparameters. 
The comparison is summarized in Table \ref{tab:table_1}.

\begin{table*}[h]\small
\centering
\setlength{\tabcolsep}{0.5mm}{
\scalebox{0.9}{
\begin{tabular}{@{}cc|cccccccc|cc@{}}

\toprule\toprule
\multirow{2}*{Environment} & \multirow{2}*{Dataset} & \multirow{2}*{SAC} & \multirow{2}*{TD3+BC} & \multirow{2}*{CQL} & \multirow{2}*{IQL} & \multirow{2}*{MOPO} & \multirow{2}*{COMBO} & \multirow{2}*{SynthER} & \multirow{2}*{Diffuser} & \multicolumn{2}{c}{ADEPT+ (Ours)} \\
~ & ~ & ~& ~& ~& ~& ~& ~& ~&~ & IQL &SAC\\
\hline
halfcheetah & \multirow{3}*{rnd} & 30.5& 11.3& 35.4&12.5 & 35.4 & 38.8 &17.2& 3.6 &$34.5\pm1.1$& $\bm{39.5 \pm 0.8}$ \\ 
walker2d & ~    & 4.1 & 0.6 & 7.0 & 5.4 & 13.6 & 7.0 & 4.2 & 3.5& $10.3\pm 2.2$ & $\bm{23.6\pm0.9}$ \\ 
hopper & ~      & 11.3& 8.6 & 10.8& 7.5 & 11.7 & 17.9 & 7.7 & 6.3& $34.4\pm0.9$& $\bm{37.5 \pm 5.7}$ \\ \hline
halfcheetah & \multirow{3}*{med} & -4.3& 48.1& 44.4&47.4 & 42.3 & 54.2 &49.6& 42.8& $55.2\pm0.5$& $\bm{56.6 \pm 2.3}$\\ 
walker2d & ~   & 0.9 & 82.7& 79.2 &78.3& 17.8 & 81.9 & 84.7 & 79.6&$\bm{97.2\pm6.5}$&$70.9\pm4.7$\\ 
hopper & ~  & 0.8 & 60.4& 58.0 &66.3& 28.0 & \textbf{97.2} & 72.0& 74.3 & $69.6\pm3.5$ &$32.8\pm3.9$\\ \hline
halfcheetah & \multirow{3}*{mix} & 29.9 & 44.8 & 46.2 &44.2 & 53.1 & 55.1 & 46.6 & 37.7 &$51.2\pm1.2$&$\bm{59.3\pm4.1}$\\ 
walker2d & ~ & 21.0 & 85.6 & 26.7 &94.7 & 39.0 & 56.0 & 83.3 & 70.6 &$\bm{105.3\pm1.4}$& $85.0\pm2.2$ \\ 
hopper & ~ & 70.4 & 64.4 & 48.6 &73.9 & 67.5& 89.5&103.2 & 93.6 &$97.6\pm3.7$& $\bm{104.0\pm2.7}$ \\ \hline
\multicolumn{2}{c|}{Average} & 18.3 & 45.2 & 39.6 & 47.8 & 55.3 & 34.3 & 52.1 & 45.8 & $\bm{61.7\pm2.3}$ & $56.6\pm3.0$ \\
\bottomrule\bottomrule
\end{tabular}}}
\caption{The evaluation of ADEPT compared with other SOTA algorithms including offline RL and diffusion-based algorithms, on D4RL MuJoCo environments with random(rnd), medium(med) and medium-replay(mix) datasets. We show the mean of standard deviation of the performance over 5 different seeds. ADEPT significantly outperforms the baselines and achieves very stable results, especially on random datasets due to its closed-loop operation supporting iterative diffusion world model adaptation and offline policy improvement.}
\label{tab:table_1}
\end{table*}

From our experimental result, the proposed ADEPT algorithms outperform the existing SOTA offline RL algorithms in most of the environments, especially in random dataset. This is consistent with our hypothesis that adapting the diffusion world model in offline RL is more critical, when there is a lack of expert demonstrations and the distribution shift becomes more severe as target policy moves toward optimum. Compared to the original SAC method which performs poorly due to lack of policy regularization, the diffusion model generated data has significantly improved its performance, with an average of 119.2\% and 104.7\% gain on random and medium-replay datasets, respectively.

This is because adaptive diffusion model guided policy improvement could be viewed as a conservative regularization method. For out-of-distribution situations, the diffusion model generates more relevant state transition results and corresponding reward function, which mitigates overestimation. For IQL that has already regularized its value function, the adaptive diffusion world model also provides synthetic data to complement the offline dataset, leading to an average of 211.8\% and 19.4\% performance gain over IQL on random and medium-replay datasets, respectively.

Compared with other diffusion-based offline RL algorithms, our method still gets significant advances in random and medium-replay datasets. Such a result shows that with a lack of expert demonstration and plenty of random data, a closed-loop iterative algorithm for diffusion world model adaptation and offline policy improvement could become more advantageous than using diffusion models as a policy representation like Diffuser~\cite{janner2022planning}, or generating synthetic dataset one time before training as SynthER~\cite{lu2023synthetic}. We also note that such a method may not be necessary when the dataset is collected by an expert policy, as less adaptation to the world model is needed.

\subsection{Ablation Study}

In the ablation study, we intend to validate the necessity of importance sampling and the effectiveness of each part in the close-loop ADEPT algorithm.
To accomplish this, we compare our methods under different settings: (1) offline RL with no generated data used, \emph{i.e.}, standard IQL and SAC; (2) offline RL with diffusion model trained one-time before training, so no importance sampling technique, \emph{i.e.}, Diff+IQL and Diff+SAC; (3) replacing policy improvement with behavior cloning to demonstrate the impact of offline RL. These three settings show the importance of adaptive diffusion world model with importance sampling, and offline reinforcement learning, respectively. We also perform additional experiments to choose the hyperparameters and validate the effectiveness of subordinate techniques like data normalization and clipping. The results are shown in the Appendix.

\begin{figure}[thp] \label{fig:ablation}
    \centering
    \includegraphics[width=0.9\columnwidth]{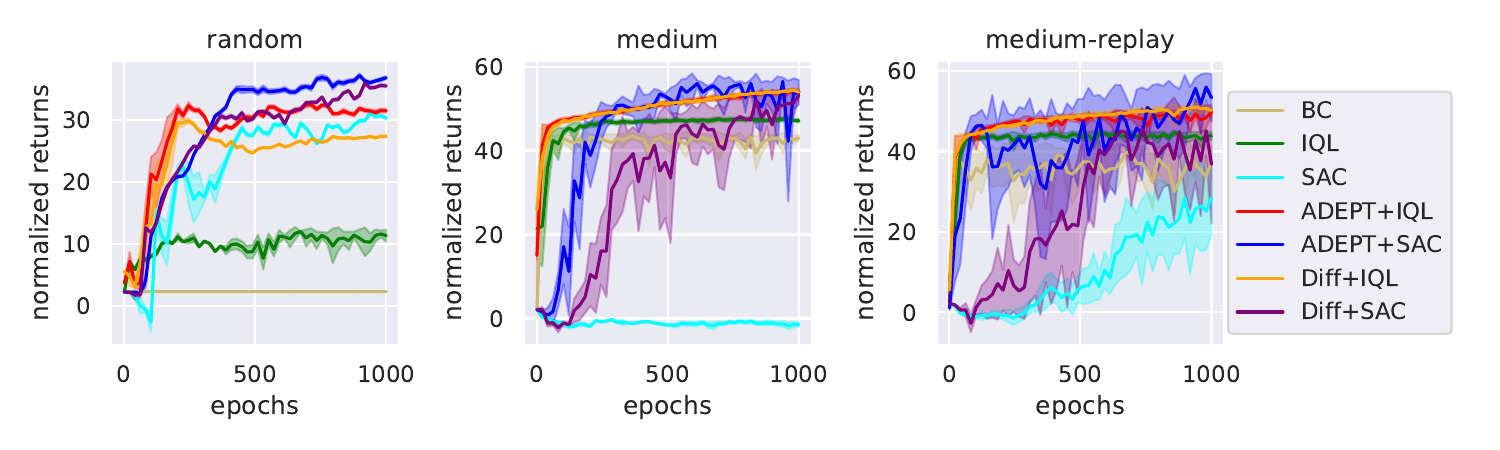}
    \caption{The training curves of different methods on halfcheetah environment for our ablation study. In particular, Diff+IQL and Diff+SAC generate synthetic data while removing the use of importance sampling, using a diffusion model trained one-time before training. IQL and SAC further remove the use of any synthetic data. All other settings remain the same with ADEPT methods. The results demonstrate the critical role of diffusion world model adaptation in offline RL. }
    \label{fig:enter-label}
\end{figure}

Figure~\ref{fig:ablation} shows the training curves of different algorithms in the halfcheetah environment in our ablation study. With no synthetic data, the original SAC algorithm fails in medium and medium-replay datasets (even lower than behavior cloning), while IQL fails in the random datasets. The diffusion world model generated samples help improve the performance of these two methods in all kinds of datasets. Besides, adopting importance sampling for diffusion model adaptation improves the performance of SAC substantially in all three datasets while only improving IQL significantly in random dataset. That is consistent with the hypothesis that samples in the random dataset are more scattered, so the importance sampling method for diffusion world model adaptation could have a higher impact on the generated distribution, aligning the world model with the target policy. The improvement on SAC shows that importance sampling reduces the distribution shift and helps SAC to reach comparable performance as IQL, even with no extra regularization methods used.

\section{Conclusion}
This paper proposes ADEPT to enable (i) a guided diffusion world
model to directly evaluate the target policy in offline reinforcement learning and (ii) an importance-sampled world model update to adaptively align the world model with the updated policy, in a closed-loop operation. Our analysis of the algorithm provides an upper bound on the return gap and illuminates key factors affecting the learning performance. Evaluations on D4RL show significantly improved results over state-of-the-art baselines. On random and medium-replay datasets, we obtain 119.2\% and 104.7\% performance improvements for SAC, and 211.8\% and 19.4\% for IQL on average. Our work provides important insights into the use of world model with offline reinforcement learning. However, the theoretical analysis and experiments are based on simple environments such as MuJoCo tasks. In practice, the three error sources could be hard to reduce, especially in complicated partial-observable environments, which is the limitation of this paper and remains to be further researched.


\bibliography{neurips_2023}
\bibliographystyle{unsrtnat}

\appendix

\section{Proof of the Theorem}

In this appendix, we show the complete proof for the return gap bound proposed in the main paper.

\begin{lemma} \label{A1}
For any 2 different joint distribution of state and action pairs $P$ and $\hat{P}$ under two policies $\pi$ and $\hat{pi}$, we have
\begin{equation*}
    D_{TV}(P(s, a)\Vert \hat{P}(s,a)) \leq D_{TV} (P(s)\Vert \hat{P}(s)) + \max_s D_{TV}(\pi(a \vert s)\Vert \hat{\pi}(a\vert s)).
\end{equation*}
\end{lemma}
\begin{proof}
\begin{align*}
    D_{TV}(P(s,a)\Vert \hat{P}(s,a)) & = \frac{1}{2}\sum_{s,a} \vert P(s,a) - \hat{P}(s,a)\vert \\
    & = \frac{1}{2} \sum_{s,a} \vert P(s) \pi(a\vert s) - \hat{P}(s,a) \hat{\pi}(a\vert s) \vert \\
    & = \frac{1}{2} \sum_{s,a} \vert P(s) \pi(a\vert s) - P(s) \hat{\pi}(a\vert s) + P(s) \hat{\pi}(a\vert s) - \hat{P}(s,a) \hat{\pi}(a\vert s) \vert \\
    & \leq \frac{1}{2} \sum_{s,a} P(s) \vert \pi(a\vert s) - \hat{\pi}(a\vert s) \vert + \frac{1}{2} \sum_{s,a} \vert P(s) - \hat{P}(s)\vert \hat{\pi}(a\vert s) \\
    & = \mathbb{E}_{s\sim P} \left[D_{TV}(\pi(a\vert s)\Vert \hat{\pi}(a\Vert s)) \right] + \frac{1}{2}\sum_s \vert P(s) - \hat{P}(s) \vert\\
    & = \mathbb{E}_{s\sim P} \left[D_{TV}(\pi(a\vert s)\Vert \hat{\pi}(a\Vert s)) \right] + D_{TV} (\sum_s \vert P(s) - \hat{P}(s))\\
    & \leq D_{TV} (\sum_s \vert P(s) - \hat{P}(s)) + \max_s D_{TV}(\pi(a\vert s) \Vert \hat{\pi} (a\vert s))
\end{align*}
\end{proof}
To be noticed that this equation can be extended to conditional probabilities:
\begin{equation*}
    D_{TV}(P(s', a'\vert s, a)\Vert \hat{P}(s',a'\vert s, a)) \leq D_{TV} (P(s'\vert s, a)\Vert \hat{P}(s'\vert s,a)) + \max_{s'} D_{TV}(\pi(a' \vert s')\Vert \hat{\pi}(a'\vert s')).
\end{equation*}

\begin{lemma} \label{A2}
    For $P_t$ under $\pi$ and $\mathcal{M}$, and $\hat{P}_t$ under $\hat{\pi}$ and $\hat{\mathcal{M}}$ at time step $t$, given $P_0(s) = \hat{P}_0(s), \forall s$, and there existing $\delta$ \emph{s.t.}
    \begin{equation}
        \max_t \mathbb{E}_{s,a \sim \pi}D_{TV}(P_{t+1}(s',a' \vert s_t = s , a_t = a) \Vert \hat{P}_{t+1}(s',a'\vert s_t = s, a_t = a)) \leq \delta.
    \end{equation}
    Then we have
    \begin{equation}
        D_{TV}(P_t(s, a) \Vert \hat{P}_t(s, a)) \leq t\delta + D_{TV}(P_0(s,a)\Vert \hat{P}_0(s,a)).
    \end{equation}
\end{lemma}
\begin{proof}
\begin{align*}
    & \left| P_t(s',a') - \hat{P}_t(s',a')\right| = \left| \sum_{s,a} \left( P_{t-1} (s,a)P(s', a'\vert s,a) - \hat{P}_{t-1}(s,a)\hat{P}(s',a'\vert s,a) \right) \right|\\
    &\leq \sum_{s,a} \left|\left( P_{t-1} (s,a)P(s', a'\vert s,a) - \hat{P}_{t-1}(s,a)\hat{P}(s',a'\vert s,a) \right)\right|\\
    &= \sum_{s,a} \left| \left( P_{t-1}(s,a)\left(P(s',a'\vert s,a) - \hat{P}(s',a'\vert s,a)\right)\right) + \hat{P}(s',a'\vert s.a)\left(P_{t-1}(s,a) - \hat{P}_{t-1}(s,a)\right)\right|\\
    &\leq \sum_{s,a} \left(P_{t-1}(s,a)\left| P(s',a'\vert s,a) - \hat{P}(s',a'\vert s.a)\right| + \hat{P}(s', a'\vert s,a)\left|P_{t-1}(s,a) - \hat{P}_{t-1}(s,a)\right| \right)\\
    & = \mathbb{E}_{s,a\sim P_{t-1}}\left|P(s',a'\vert s,a) - \hat{P}(s',a'|s,a)\right| + \sum_{s,a}\hat{P}(s',a'\vert s,a)\left|P_{t-1}(s,a) - \hat{P}_{t-1}(s,a)\right|
\end{align*}
Therefore, we have:
\begin{align*}
& D_{TV}(P_t(s,a)\Vert \hat{P}_t (s,a)) = \frac{1}{2} \sum_{s',a'}\left|P_t(s',a') - \hat{P}_t(s',a')\right| \\
& \leq \frac{1}{2}\sum_{s',a'} \left( \mathbb{E}_{s,a\sim P_{t-1}}\left|P(s',a'\vert s,a) - \hat{P}(s',a'|s,a)\right| + \sum_{s,a}\hat{P}(s',a'\vert s,a)\left|P_{t-1}(s,a) - \hat{P}_{t-1}(s,a)\right| \right)\\
& = \mathbb{E}_{s,a \sim P_{t-1}}D_{TV}(P(s',a' \vert s , a) \Vert \hat{P}(s',a'\vert s, a)) +  \frac{1}{2}\sum_{s',a'}\sum_{s,a}\hat{P}(s',a'\vert s,a)\left|P_{t-1}(s,a) - \hat{P}_{t-1}(s,a)\right|\\
& \leq \delta + \frac{1}{2} \sum_{s,a}\left|P_{t-1}(s,a) - \hat{P}_{t-1}(s,a)\right|
= \delta + D_{TV}\left(P_{t-1}(s,a)\Vert \hat{P}_{t-1}(s,a)\right)\\
&\leq t\delta + D_{TV}(P_{0}(s,a)\Vert \hat{P}_{0}(s,a))
\end{align*}
Furthermore, if the initial state distribution is the same for the two environmental model, \emph{i.e.,} $P_0(s) = \hat{P}_0(s), \forall s$, then by applying Lemma \ref{A1} we have:
\begin{align*}
    D_{TV}(P_t(s,a)\Vert \hat{P}_t (s,a)) & \leq t\delta + D_{TV} (P_0(s)\Vert \hat{P}_0(s)) + \max_s D_{TV}(\pi(a \vert s)\Vert \hat{\pi}(a\vert s))\\
    &  \leq t\delta + \max_s D_{TV}(\pi(a \vert s)\Vert \hat{\pi}(a\vert s))
\end{align*}

\end{proof}

\begin{lemma} \label{A3}
    We define $r_{max} = \max_{s,a} r(s,a)$. For any $\pi, \hat{\pi}$ under any two environmental model $\mathcal{M}$ and $\hat{\mathcal{M}}$, satisfying
    \begin{equation}
        \max_t \mathbb{E}_{s_t,a_t \sim \pi} D_{TV}(P(r_t\vert s_t,a_t) \Vert \hat{P}(r_t \vert s_t,a_t)) \leq \varepsilon_r ,
    \end{equation}
    we have 
    \begin{equation}
        \left\vert J(\pi) - \hat{J}(\hat{\pi}) \right\vert
        \leq r_{max} \sum_t \gamma^t (2\varepsilon_r + 2D_{TV}(P_t(s,a)\Vert \hat{P}_t(s,a))).
    \end{equation}
\end{lemma}
\begin{proof}
\begin{align*}
    & \left|J(\pi) - \hat{J}(\hat{\pi})\right| = \left|\sum_t \gamma^t \sum_{s,a}\left( P_t(s,a)r(s,a) - \hat{P}_t(s,a)\hat{r}(s,a)\right)\right|\\
    & = \left| \sum_t \gamma^t \left( \sum_{s,a}\left[P_t(s,a)\left(r(s,a) - \hat{r}(s,a)\right) + \hat{r}(s,a)\left(P_t(s,a) - \hat{P}_t(s,a)\right)\right]  \right)\right|\\
    &\leq \sum_t \gamma^t \left( \left| \sum_{s,a,r}r P_t(s,a)\left[ P_(r|s,a) - \hat{P}(r\vert s,a) \right] \right| + r_{max} \sum_{s,a}\left| P_t(s,a) - \hat{P}_t(s,a) \right| \right)\\
    &\leq \sum_t \gamma^t r_{max}\left( \mathbb{E}_{s,a\sim \pi} \sum_r \left| P(r\vert s,a) - \hat{P}(r\vert s,a)\right|\right)\\
    &\leq r_{max} \sum_t \gamma^t \left(2\varepsilon_r + 2 D_{TV}(P_t(s,a)\Vert \hat{P}_t(s,a))\right)
\end{align*}
\end{proof}

\begin{theorem} Given $\hat{\varepsilon}_r$, $\hat{\varepsilon}_m$ and $\hat{\varepsilon}_\pi$ with the Definition \ref{definition_1}, \ref{definition_2} and \ref{definition_3}, the bound $C$ between the true return $J$ and the IS-diffusion model-based return $\hat{J}_{\theta, \eta}$ under a same evaluating policy $\pi$ can be expressed as follows:
\begin{equation}
    J(\pi) \geq \hat{J}_{\theta, \eta}(\pi) - 2r_{max}(\frac{\hat{\varepsilon}_r + 2\hat{\varepsilon}_\pi}{1 - \gamma} + \frac{\gamma(2\hat{\varepsilon}_\pi + \hat{\varepsilon}_m)}{(1-\gamma)^2}).
\end{equation}
\end{theorem}
\begin{proof}
We denote $\pi_\mathcal{D}$ as the policy collecting the trajectories in the diffusion world model. The return gap could be separated into:
\begin{equation*}
    \vert J(\pi) - \hat{J}_{\theta, \eta}(\pi) \vert \leq \vert J(\pi) - J(\pi_\mathcal{D})\vert + \vert J(\pi_\mathcal{D}) - \hat{J}_{\theta, \eta}(\pi_\mathcal{D})\vert + \vert \hat{J}_{\theta, \eta}(\pi_\mathcal{D}) - \hat{J}_{\theta, \eta}(\pi)\vert
\end{equation*}
By applying Lemma \ref{A3}, we have:
\begin{equation*}
    \left| J(\pi) - J(\pi_\mathcal{D})\right| \leq r_{max} \sum_t \gamma^t \left(2\hat{\varepsilon}_r + 2 D_{TV}(P_t(s,a)\Vert \hat{P_t}(s,a)) \right)
\end{equation*}
Considering the condition in Lemma \ref{A2}, according to Lemma \ref{A1} it can be bounded by:
\begin{align*}
    & \max_t \mathbb{E}_{s,a \sim \pi}D_{TV}(P_{t+1}(s',a' \vert s , a) \Vert \hat{P}_{t+1}(s',a'\vert s, a))\\
    & \leq \max_t \mathbb{E}_{s,a \sim \pi} D_{TV} (P(s'\vert s, a)\Vert \hat{P}(s'\vert s,a)) + \max_{s} D_{TV}(\pi(a \vert s)\Vert \hat{\pi}(a\vert s))
\end{align*}

Since in this term the environmental model is the same while policy is variant, the state transition error and reward error are zero, and we have:
\begin{equation*}
    \max_t \mathbb{E}_{s,a \sim \pi}D_{TV}(P_{t+1}(s',a' \vert s , a) \Vert \hat{P}_{t+1}(s',a'\vert s, a)) \leq \hat{\varepsilon}_\pi
\end{equation*}
Next, we replace $\delta$ with $\hat{\varepsilon}_\pi$ in Lemma \ref{A2} and get:
\begin{align*}
    \left| J(\pi) - J(\pi_\mathcal{D})\right| & \leq 2r_{max} \sum_t \gamma^t D_{TV}(P_t(s,a)\Vert \hat{P_t}(s,a)) \\
    & \leq 2r_{max} \sum_t \gamma^t (t+1)\hat{\varepsilon}_\pi
\end{align*}
The third term could be analysed similarly, and we get:
\begin{align*}
    \left| \hat{J}_{\theta, \eta}(\pi_\mathcal{D}) - \hat{J}_{\theta, \eta}(\pi) \right| \leq 2r_{max} \sum_t \gamma^t (t+1)\hat{\varepsilon}_\pi
\end{align*}
For the second term referring to the return gap of the behavior policy under the real environment and model, there's no policy shift but having the state transition error:
\begin{equation*}
    \max_t \mathbb{E}_{s,a \sim \pi}D_{TV}(P_{t+1}(s',a' \vert s , a) \Vert \hat{P}_{t+1}(s',a'\vert s, a)) \leq \hat{\varepsilon}_m
\end{equation*}
By replacing $\delta$ with $\hat{\varepsilon}_m$ in Lemma \ref{A2} we get:
\begin{equation*}
    \left| J(\pi_\mathcal{D}) - \hat{J}_{\theta, \eta}(\pi_\mathcal{D}) \right| \leq 2r_{max} \sum_t \gamma^t \left(\hat{\varepsilon}_r + (t+1)\hat{\varepsilon}_m \right)
\end{equation*}
Finally, we summed the bounds of all the three terms and get:
\begin{equation*}
    \vert J(\pi) - \hat{J}_{\theta, \eta}(\pi) \vert \leq 2 r_{max} \sum_t \gamma^t \left(\hat{\varepsilon}_r + (t+1)\hat{\varepsilon}_m + 2 (t+1)\hat{\varepsilon}_\pi \right) = 2r_{max}(\frac{\hat{\varepsilon}_r + 2\hat{\varepsilon}_\pi}{1 - \gamma} + \frac{\gamma(2\hat{\varepsilon}_\pi + \hat{\varepsilon}_m)}{(1-\gamma)^2})
\end{equation*}

\end{proof}

\section{Experiment Details}

We use D4RL \cite{fu2020d4rl} datasets for evaluation, their code could be found at \url{https://github.com/Farama-Foundation/D4RL}. Their datasets are licensed under the Creative Commons Attribution 4.0 License (CC BY), and their code is licensed under the Apache 2.0 License.

\subsection{Baselines}

We select a number of SOTA baselines algorithms, including model-free methods IQL\cite{kostrikov2021offline}, SAC \cite{haarnoja2018soft}, TD3+BC \cite{fujimoto2021minimalist}, CQL \cite{kumar2020conservative}, model-based methods such as MOPO \cite{yu2020mopo}, COMBO \cite{yu2021combo}, and diffusion-based methods as SyntheER \cite{lu2023synthetic} and Diffuser \cite{janner2022planning}.
We run the IQL and SAC code for evaluation to get the result, while other results on D4RL dataset are obtained from the original paper of each method. Specially, Diffuser doesn't report their results in random dataset, therefore we run its code from \url{https://github.com/jannerm/diffuser} for evaluation.
In addition, we cite to the following works. \cite{zhang2024collaborative, zou2024distributed, gao2024cooperative, zhang2024modeling, wang2023osteoporotic, fang2022coordinate, fang2023implementing, zhou2022pac, zhou2023value, mei2023mac, chen2023real, chen2021bringing}

\subsection{Computational Resources and Costs}
All of the experiments in this paper are conducted on a server with an AMD EPYC 7513 32-Core Processor CPU and an NVIDIA RTX A6000 GPU. The training of the diffusion model costs approximately 1 hour for 1M gradient steps. The offline training with importance-sampling costs nearly 5 hours for 1M gradient steps.
\subsection{Hyperparameter Settings}

We show the hyperparameters used in the training process in Table \ref{H1} and \ref{H2}. These hyperparameters are shared in all of the environments.

\begin{table*}[h] \label{H1}
    \centering
    \setlength{\tabcolsep}{40pt}
    \begin{tabular}[width=\columnwidth]{@{}cc@{}}
    \toprule\toprule
        Parameter & Value \\
    \hline
        denoising steps & 10 \\
        $s$ for cosine schedule & $1\times 10^{-4}$\\
        train batchsize & 256\\
        importance-sampling batchsize & 1024 \\
        learning rate & $3\times 10^{-4}$\\
        optimizer & Adam\\
        hidden dimension & 256\\
        error model depth & 3 \\
        model training steps & $1\times 10^{6}$ \\
        Horizon & 10 \\
        condition guidance weight & 0.1\\
    \bottomrule \bottomrule
    \end{tabular}
    \caption{Diffusion Training Hyperparameters}
    \label{tab:my_label}
\end{table*}

\begin{table*}[h] \label{H2}
    \centering
    \setlength{\tabcolsep}{40pt}
    \begin{tabular}[width=\columnwidth]{@{}cc@{}}
    \toprule\toprule
        Parameter & Value \\
    \hline
        $\gamma$ & 0.99 \\
        learning rate & $3\times 10^{-4}$ \\
        train batch & 256 \\
        replay buffer size & $1 \times 10^{6}$ \\
        evaluation steps per epoch & $1000$ \\
        gradient steps per epoch & $1000$ \\
        training epochs & $1000$ \\
        optimizer & Adam \\
        hidden dimension & 256 \\
        soft target updata rate (SAC) & $5 \times 10^{-3}$ \\
        $\alpha$ (IQL) & 0.005 \\
        $\tau$ (IQL) & 0.7 \\
        $\beta$ (IQL) & 3.0 \\
    \bottomrule \bottomrule
    \end{tabular}
    \caption{Offline RL Hyperparameters}
    \label{tab:my_label}
\end{table*}

\section{Additional Ablation Studies}
In ADEPT, high values of the denoising steps $K$ and Importance-sampling batchszie $B_m$ generally has higher accuracy on the state prediction and better performance in model adjustment since it could cover more samples in the datasets. However, increasing $K$ will extends the training procedure since the diffusion model executes more denoising steps in state generation. Also, a large $H$ decelerates the training of the RL agent, since there are more policy inference steps in each updating.
\begin{table*}[h] \label{den}
\centering
\setlength{\tabcolsep}{1mm}{
\begin{tabular}{l|c|cc|cc|c}
\toprule\toprule
Environment & baseline & $K = 5$ & $K = 20$ & $B_m = 256$ &$B_m = 2048$&no clipping\\ \hline
halfcheetah-random &39.5&30.0&33.2&42.1&40.5&18.52\\
walker2d-random &23.6&24.0&21.5&8.2&23.1&13.32\\
hopper-random &37.5&33.0&34.7&32.8&34.4&32.7\\
halfcheetah-medium &56.6&55.4&60.4&56.1&55.2&2.9\\
walker2d-medium &70.9&52.7&73.7&22.6&66.0&21.7\\
hopper-medium &32.8&4.4&30.1&4.6&31.7&2.7\\
halfcheetah-medium-replay &59.3&49.5&57.1&53.6&59.3&46\\
walker2d-medium-replay &85.5& 86.4 & 88.6 & 22.1& 82.4 &16.4\\
hopper-medium-replay &104.0 &77.8&104.6 &98.7&99.1 & 41.0\\
\bottomrule \bottomrule
\end{tabular}}
\caption{Ablation study for denoising steps}
\end{table*}

We conducted additional ablation study on the value of $K$ and $B_m$, and the use of clipping technique. The results are shown in the Table \ref{den}. Based on these results we can conclude that for simple environments like MuJoCo tasks, improving $K$ from 10 to 20 doesn't change much on performance, while the evaluation time is nearly doubled. If $K$ is too low, the estimation error can be large. The selection of $B_m$ could have similar influence on the performance as $K$. Therefore, we select $K$ to be 10 and $B_m$ to be 1024 in our method. Besides, clipping of the generated state and reward have a significant affects on performance improvement while avoiding out-of-distribution outputs.


\end{document}